\documentclass[journal]{IEEEtran}
\ifCLASSINFOpdf
\else
\fi

\usepackage{times}
\usepackage{graphicx}
\usepackage{amsmath}
\usepackage{amssymb}
\usepackage{amsthm}
\usepackage{amsfonts}
\usepackage{booktabs}
\usepackage{siunitx}
\usepackage{multirow}
\usepackage[caption=false]{subfig}
\usepackage{subfloat}
\usepackage[pagebackref=true,breaklinks=true,letterpaper=true,colorlinks, linkcolor=red,bookmarks=false]{hyperref}
\newtheorem{thm}{Theorem}
\newtheorem{prop}{Proposition}

\begin{document}

\title{Invariant Deep Compressible Covariance Pooling for Aerial Scene Categorization}

\author{Shidong~Wang, Yi~Ren, Gerard Parr \IEEEmembership{Member,~IEEE}
	Yu~Guan and~Ling~Shao, \IEEEmembership{Senior Member,~IEEE}
	\thanks{Manuscript received May 1, 2020; revised June 24, 2020 and August 11, 2020; accepted September 15, 2020. This work was supported by the EPSRC DERC: Digital Economy Research Centre under Grant EP/M023001/1. \textit{(Corresponding author: Yu Guan.)}}
	\thanks{Shidong Wang is with Open Lab, School of Computing, Newcastle University, Newcastle upon Tyne NE4 5TG, U.K., also with the School of Computing Sciences, University of East Anglia, Norwich NR4 7TJ, U.K. (e-mail: shidong.wang@ncl.ac.uk)}
	\thanks{Yi Ren and Gerard Parr are with School of Computing Sciences, University of East Anglia, Norwich, U.K. (e-mail: e.ren and g.parr@uea.ac.uk)}
	\thanks{Yu Guan is with Open Lab, School of Computing, Newcastle University, Newcastle upon Tyne NE4 5TG, U.K. (e-mail: yu.guan@ncl.ac.uk).}
	\thanks{Ling Shao is with Inception Institute of Artificial Intelligence, Abu Dhabi, United Arab Emirates, e-mail:(ling.shao@ieee.org)}}

\markboth{Journal of IEEE Transactions on Geoscience and Remote Sensing, October~2020}%
{Shell \MakeLowercase{\textit{et al.}}: Bare Demo of IEEEtran.cls for IEEE Journals}

\maketitle

\begin{abstract}
Learning discriminative and invariant feature representation is the key to visual image categorization. In this article, we propose a novel invariant deep compressible covariance pooling (IDCCP) to solve nuisance variations in aerial scene categorization. We consider transforming the input image according to a finite transformation group that consists of multiple confounding orthogonal matrices, such as the D4 group. Then, we adopt a Siamese-style network to transfer the group structure to the representation space, where we can derive a trivial representation that is invariant under the group action. The linear classifier trained with trivial representation will also be possessed with invariance. To further improve the discriminative power of representation, we extend the representation to the tensor space while imposing orthogonal constraints on the transformation matrix to effectively reduce feature dimensions. We conduct extensive experiments on the publicly released aerial scene image data sets and demonstrate the superiority of this method compared with state-of-the-art methods. In particular, with using ResNet architecture, our IDCCP model can reduce the dimension of the tensor representation by about 98\% without sacrificing accuracy (i.e., \textless 0.5\%).
\end{abstract}


\IEEEpeerreviewmaketitle

\section{Introduction}\label{introduction}
Aerial scene classification, also known as remote sensing scene classification, is considered to be one of the most active tasks in scene classification. The classification of aerial scene images involves a wide range of applications, such as environment monitoring, urban and agricultural planning, and land use and land cover (LULC) classification \cite{dataset:xia2017aid,dataset:cheng2017remote,mg_cap,li2019deep}. The recent development of remote sensing technologies leads to the accumulation of very high spatial resolution images (e.g., $\sim$1-4 m/pixel), which takes out aerial imagery characteristics to the new level of illustrating the geometry structure and texture peculiarities in a more distinct way. The increasingly spatial resolution of aerial images not only allows depicting image peculiarities across smaller spatial extents but also makes the classification more ambiguous and challenging.

\begin{figure}[]
	\centering
	\includegraphics[width=0.48\textwidth]{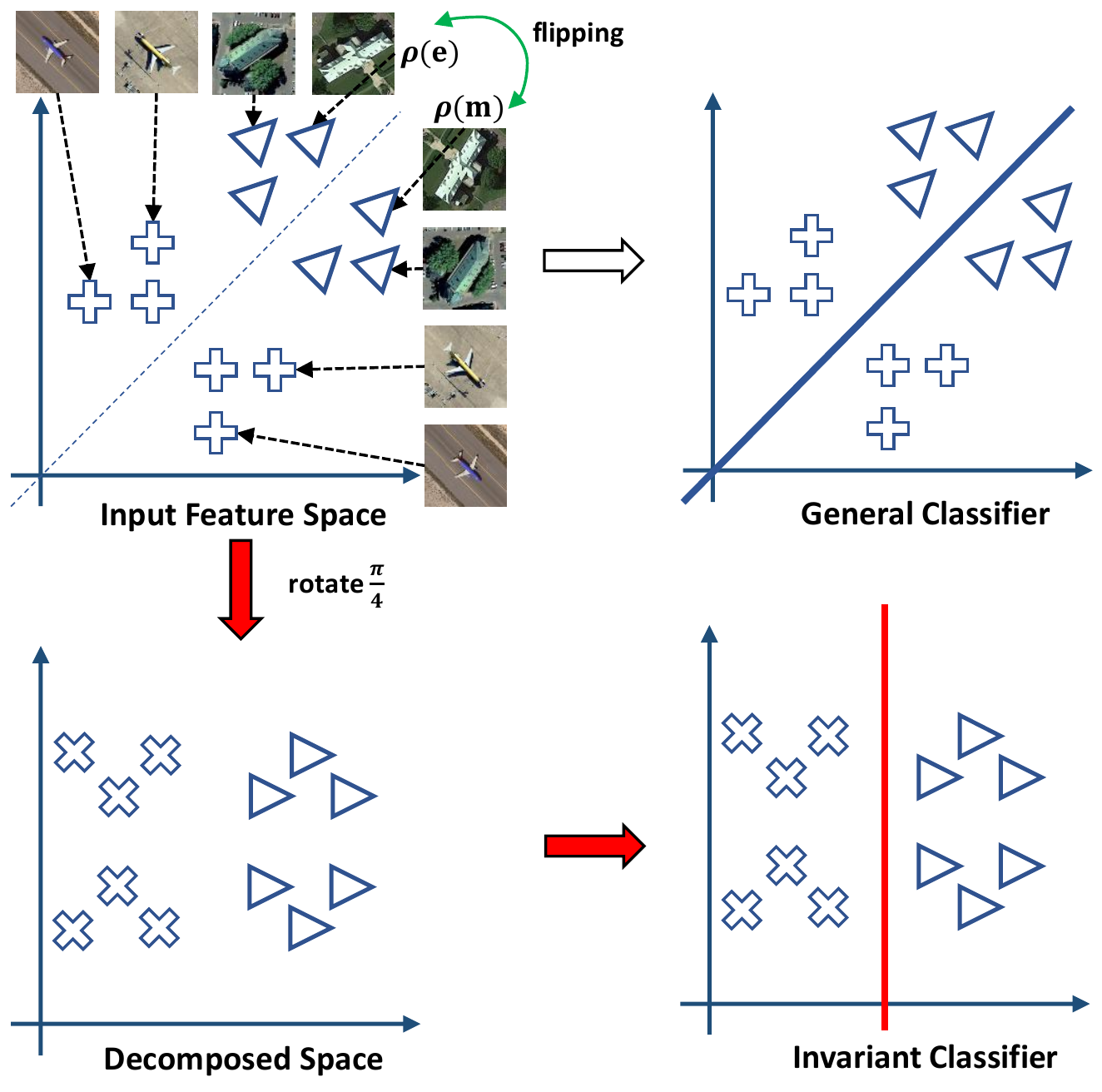}
	\caption{{Solely flipping the input image may make the conventional classifier inoperable. Combined with the rotation transformation, a new orthogonal representation space is formed. Then, we can infer a trivial representation from the space and leverage it to train our invariant classifier.}}
	\label{idccp:idccp_intro}
\end{figure}
Compared with conventional scene images, the texture
information of remote sensing images is more complicated.
The main reason for sophisticated texture features is the
variation of orientation, scale, and shape of objects in the image. In addition to these variations, the inherent property of remote sensing images is also quite different from the ordinary scene images. Precisely, remote sensing image, as one of the most representative overhead images, has no dominant left–right or up–down relationships. To classify a typical scene image, we only concern the absence or presence of the main object. However, in the aerial scene classification task, an expectation is that the model is capable of assigning the correct label for a given image regardless of its absolute orientation. This sought-after property remains strictly constant under all transformations of the input data, which is so-called \textit{invariance}.                 

Invariance can be directly encoded and considered to be the most effective method to mitigate the impact of variations of the input data. However, incorporating invariant information is challenging, even for the powerful convolutional neural network (CNN) architectures. Precisely, off-the-shelf CNN architectures are only endowed with the minimal internal structures
due to the costly computing of the optimization. These minimal intrinsic structures are capable of handling locally minor shifts but not global transformations. Data augmentation \cite{dataaugmentation} is widely adopted to incorporate the prior knowledge of input data, but there is no guarantee that the invariance learned in the training stage is effectively generalized for the test data. Furthermore, it is difficult to quantify the predominate transformations and lacks the interpretability of feature maps. In contrast to the redundant approaches, such as data augmentation, one of the latest research lines is toward procuring the equivariance from equivariant CNNs \cite{group:cohen2016group,group:cohen2016steerable,group:dieleman2016exploiting,group:henriques2017warped}. The basic idea of these methods is to learn the \textit{transformation-equivariant CNN} by constructing features in a linear \textit{G}-space and then derive an invariant subspace by employing the appropriate pooling method ({e.g.,} the coset pooling). These methods can detect co-occurrences of features at any \textit{positions} in a standard CNN architecture, and any preferred \textit{poses} in a \textit{G}-space, but the computational cost scales dramatically with the increasing cardinality of the group.

To address the shortcomings of the aforementioned approaches, a novel framework is proposed to derive the transformation-invariant subspace from a finite linear \textit{G}-group space, which allows group actions to be directly applied to the raw image. As shown in Fig. \ref{idccp:idccp_intro}, merely flipping the
local feature space can render the traditional classifier fail to work. Through looking insight into the flipping operation, we find that it can be expressed by the permutation matrices. The expression of permutation matrices implies two primary properties: the flipping operation acts orthogonally at the local pixel and prevents images from distortion during transformation. These motivate us to construct a transformation group \textit{G} where all the decomposed spaces are orthogonal to each other ({i.e.,} D4 group in our scenarios). An invariant feature space can be sought through using the reducible decomposition of the representations of \textit{G}-space. Namely, it allows us to decompose the action of \textit{G} into the direct sum of irreducible
representations and results in a locally invariant subspace that serves to train an invariant classifier.   

The orthogonal transformations prevent the pixel value
shifting in the process of transforming but cannot avoid the changes of pixel locations. To alleviate the effect of pixel position changes, we can calculate the tensor product of irreducible representations to form a global representation based on the fact that the reducible decomposition of the representation conforms to the group action of \textit{G}-space. The tensor representation contains more discriminative information than the conventional first-order feature but suffers from the high-dimensional problem. Considering that the second-order feature representation is a covariance matrix ({i.e.,} symmetric positive definite (SPD) matrix), we can force the weight matrix
to be a row full-rank matrix where all elements reside on a Stiefel manifold. In this way, it can produce a compact space while maintaining the geometry of the SPD manifold. Our contributions can be summarized as follows:    
\begin{itemize}    
	\item[-] We derive an invariant classifier from the learned weights of the trivial tensor representation with the guarantee of being invariant under the finite \textit{G}-group actions.    
	\item[-] We introduce a way of imposing orthogonal constraints on the weight matrix to effectively map the high-dimensional SPD manifolds into new compact manifolds.
	\item[-] We conduct extensive experiments on four aerial scene image datasets and achieve state-of-the-art performance.
\end{itemize}                                                          
\section{Related Works}
\subsection{Aerial Scene Classification}
The last decade has witnessed a dramatic growth of research interests in aerial scene image classification. Early attempts heavily rely on the manually designed features. Globally handcrafted features summarize the overall statistical properties and can be directly fed into classifiers. Examples include color descriptors \cite{hf_global:dos2010evaluating,hf_global:li2010object,hf_global:penatti2015deep}, texture descriptors \cite{lbp:ojala2002multiresolution,lbp:ren2015learning,lbp:huang2016remote} and histograms of oriented gradients (HOG) descriptors \cite{hf_hog:cheng2015effective,hf_hog:cheng2015learning}. Locally handcrafted features usually need to be transformed into the higher level representation using coding methods, such as the Bag-of-Visual-Words (BoVW) models \cite{dataset:yang2010bag,bovw:zhao2014land,bovw:chen2014pyramid} and the Fisher vector \cite{fv:zhao2016fisher,fv:wu2016hierarchical,lbp:huang2016remote}. Shortly, unsupervised learning methods \cite{usl:hu2015unsupervised,usl:yu2017unsupervised,usl:lu2017remote,dataset:cheng2017remote,usl:fan2017unsupervised} become increasingly popular to remedy the limitations of handcrafted features. Typical unsupervised feature learning methods include, but not limited to, principal component analysis (PCA) \cite{feat:pca}, sparse coding \cite{sparse_coding:zheng2012automatic,sparse_coding:cheriyadat2014unsupervised,sparse_coding:qi2016sparse}, autoencoder \cite{autoencoder:zhang2015saliency,autoencoder:othman2016using,autoencoder:ma2016spectral,autoencoder:li2016stacked} and K-means clustering that often associates with BoVW methods. Benefiting from the capability of
incorporating intricate structures hidden in high-dimensional data, deep learning-based methods present extensive popularity of being adopted for aerial image classification. For example, Cheng \textit{et al.} \cite{bovw:cheng2017remote} investigated the efficiency of the BoVW model utilizing convolutional features. \cite{cheng2016learning} learned the rotation-invariant feature to improve the performance of object detection in remotely sensed images. Other works include transferring learning-based methods \cite{tl:hu2015transferring,dl:castelluccio2015land,tl:marmanis2015deep,tl:cheng2016scene,dl:nogueira2017towards,tl:xie2019scale}, domain adaption \cite{tl:othman2017domain} and feature fusion \cite{fusion:chaib2017deep}. Feature fusion-based methods can also achieve gratifying results and have different forms that include multi-scale \cite{fusion:zhao2016scene},multi-layer \cite{fusion:li2017integrating}, multi-stream \cite{fusion:chaib2017deep,rtn} and multi-granularity \cite{mg_cap}. For more details, we refer readers to \cite{dataset:xia2017aid} and \cite{dataset:cheng2017remote}, and a comprehensive review of remote sensing image interpretation based on deep learning techniques \cite{li2019deep}.   
\subsection{Equivariant/Invariant CNN}
Existing methods capture the transformation-invariant information by transforming the inputs or the filters. For the former category, the most standard method is data augmentation \cite{dataaugmentation}. This approach aims to increase the capacity of the model in terms of some specific variations by generating abundant
training samples. RIFD-CNN \cite{group:cheng2016rifd} introduced an explicit regularization as a constraint that forces the model learning the invariance of CNN features. Ti-pooling \cite{group:laptev2016ti} employed a parallel Siamese architecture to extract features from multiple rotated images and provided a pooling across the different features at the first fully-connected layer. Henriques and
Vedaldi \cite{group:henriques2017warped} proved that the inherent translation equivariance
of CNNs could be worked on the warped images and achieved
the transformation equivariance. Recently, Chen \textit{et al.} \cite{rtn} proposed a recurrent transformer network (RTN), which exploited the spatial transformer network progressively and learned multiple transformation-invariant scales of the input image. For the latter category, existing methods usually construct features using group representations. For example, Cohen and Welling \cite{group:cohen2016group} learned feature invariance by using the coset pooling for an asymmetry transformation group space that is composed of dihedral flipping and four $ 90^{\circ} $ rotations. This work had been extended in \cite{group:cohen2016steerable} to decouple the computational cost by performing the general, steerable representations. Instead of rotating filters, Dieleman \textit{et al.} \cite{group:dieleman2016exploiting} proposed to rotate feature maps at four $ 90^{\circ} $ rotations and preserved equivariance in CNNs. In addition, Mukuta and Harada \cite{mukuta2019invariant} and Sokolic \textit{et al.} \cite{sokolic2017generalization} analyzed the significance of encoding first- or second-order invariant features in learning algorithms.            
\subsection{Second-order Statistics Feature}
Bilinear pooling \cite{sos:lin2015bilinear}, as one of the most successful second-order pooling methods, collects second-order statistics of local CNN features over the entire image to form a holistic representation. DeepO$_{2}$P \cite{sos:ionescu2015matrix} performed a matrix back-propagation structure for both singular value decomposition (SVD) and eigenvalue decomposition (EIG). Improved bilinear pooling \cite{sos:lin2017improved} investigated the performance of using the combination
of different normalization methods, such as the matrix square root normalization, an element-wise square root, and $ {l}_2 $ normalisation. Acharya \textit{et al.} \cite{sos:acharya2018covariance} proposed a covariance pooling framework that exploited the Riemannian manifold for facial expression recognition.  Various methods have been proposed to reduce the high-dimensional of the bilinear feature, for example, Random Maclaurin \cite{sos:gao2016compact}, Tensor Sketch \cite{sos:gao2016compact}, low-rank constraints \cite{sos:kong2017low}, Gaussian RBF kernel \cite{sos:cui2017kernel} and Grassmann manifold \cite{sos:wei2018grassmann}. Besides, iSQRT-COV \cite{sos:li2018towards} provided a method to speed up the calculating of the square root of the global covariance matrix by using Newton-Schulz iteration in both forward and backward propagations.
\section{Preliminary Notions and Definitions}  
We use calligraphic typeface $\mathcal{X}$ and $\mathcal{F}$ to denote the input image and the deep CNN features, respectively. A group $G=(\mathcal{X}, \bullet)$ is the pair of a set $\mathcal{X}$, together with an operation $\bullet:\mathcal{X} \times \mathcal{X} \to \mathcal{X}$ (also known as group law) that satisfies the group axioms of closure, associativity, identity and invertibility. The number of elements in a finite $ \mathcal{X} $ is denoted as $ |\mathcal{X}| $. A homomorphism is a map from a group \textit{G} to the group of automorphisms of a vector space \textit{V} that preserves group action operations, $\rho(g_{1}) \bullet \rho(g_{2})=\rho(g_{1} \bullet g_{2}), \forall {g_1},{g_2} \in G$ and exists the \textit{d}-dimensional identity matrix $\rho(e)=1_{d\times d}$. For a concrete example, $\rho:G \to$ GL$(V)$ is a homomorphism and also called a representation, where GL is the general linear space. A representation is named a trivial representation if and only if it maps all $g \in G$ to $1_{d\times d}$ (e.g., one-dimensional trivial representation is denoted as \textbf{1}). Similarly, the representation is called a unitary representation or orthogonal representation when all $\rho(g)$ are unitary matrices or orthogonal matrices. The space of intertwining operator is written as Hom$_{\mathcal{X}}(\rho,\rho^{'})$ which implies that there is a linear operator $L:\mathbb{C}^{d} \to \mathbb{C}^{d^{'}}$ that satisfies $ L \bullet \rho(g)=\rho{'}(g) \bullet L$. If \textit{L} is a bijective function that satisfies $L \in $ Hom$_{\mathcal{X}}(\rho,\rho^{'})$, we will write it as $ \rho \simeq \rho^{'} $. Given two representations $(\rho, V_{1})$ and $(\sigma, V_{2})$ of the same group \textit{G}, the direct sum of these two representations is given as $\rho \oplus \sigma:G \to$GL$(V_{1} \oplus V_{2})$ with regarding G as block-diagonal form of $G \times G$. According to Schur's Lemma, Hom$_{G}(\rho_{1},\rho_{2})=\left\lbrace 0 \right\rbrace$ if $\rho_{1}$ and $\rho_{2}$ are not isomorphic or 1-D when they are isomorphic. If $\rho$ and $\sigma$ are in tensor spaces, the tensor representation will be denoted as $\rho \otimes \sigma$. The character function $\mathcal{X}_{\rho}$ that maps \textit{G} into a finite-dimensional vector space over a filed \textit{F} is given by $\mathcal{X}_{\rho}(g)=\text{tr}(\rho(g))$, where $\text{tr}(\cdot)$ is the trace operation. The degree of a representation $\rho$ is the dimension of its representation space \textit{V} and we denote it as dim($\rho$).          
\section{Method}
\begin{figure*}[]
	\centering
	\includegraphics[width=0.98\textwidth]{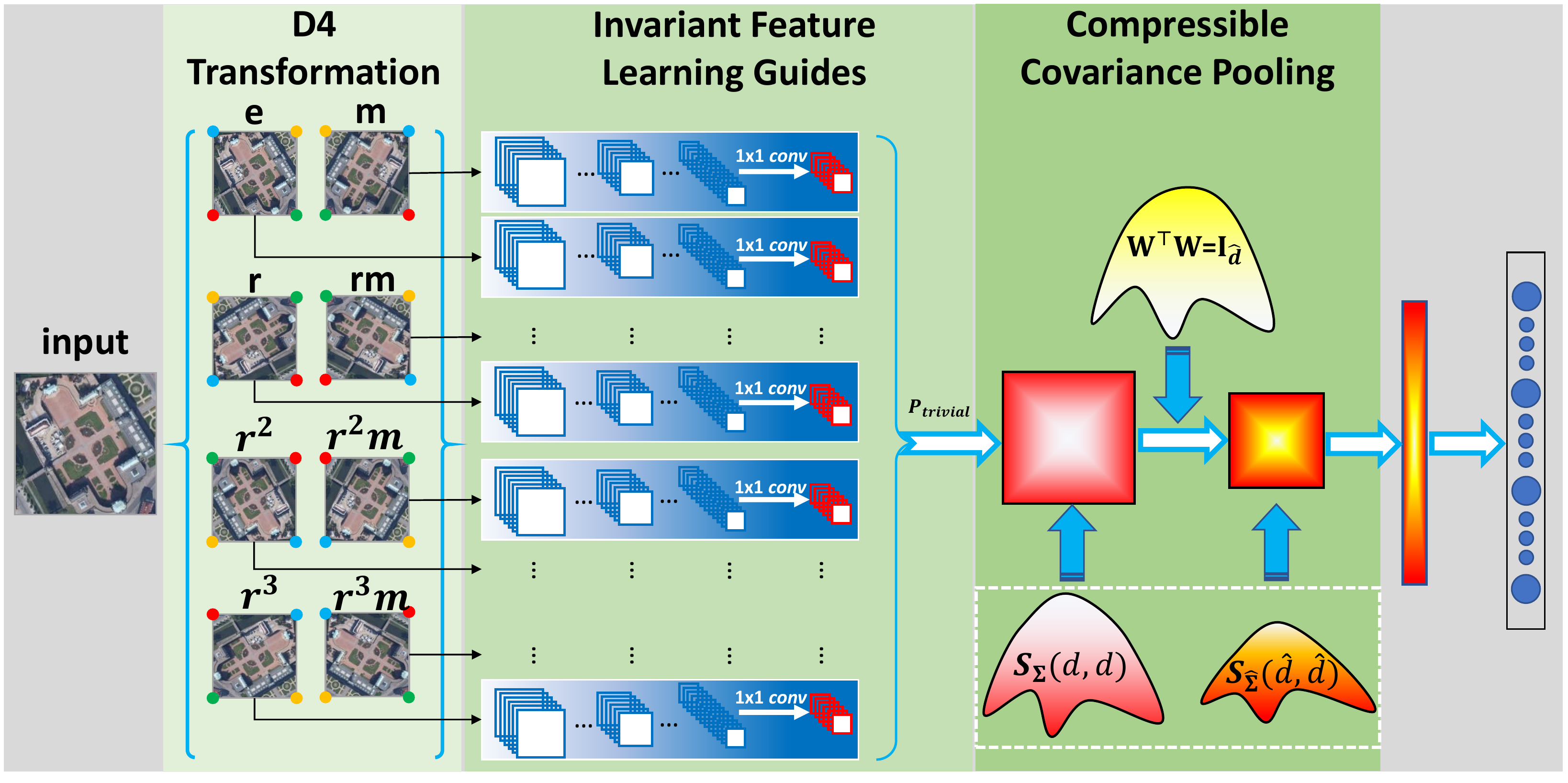}
	\caption{{An overview of the proposed IDCCP architecture. Given an input image, it will be used to generate multiple copies according to the D4 principle. Then, each copy will be fed into a subnetwork of Siamese-style CNNs to extract feature (Note:\textit{1$ \times $1 conv} is only adopted in the Siamese architecture with ResNet50 as the backbone). $ P_{trivial} $ is the projection layer to produce a trivial representation. Subsequently, orthogonal weights are adopted to compress high-dimensional manifold $ S_{\Sigma}(d,d) $ to a compact manifold $ S_{\hat\Sigma}(\hat{d},\hat{d}) $. The resulting features will be flattened and fed into the classifier to generate predictions.}}
	\label{idccp:idccp}
\end{figure*}
\subsection{Transformation-Equivariant Networks} 
In deep learning models, the transformation-equivariant
preserves the capacity to capture various useful transformations. An example is the translation-equivariant in convolution
layers, which can be exploited in any layers of the deep
CNN architecture. Given an input image $\mathcal{X}$, the transformation equivariant can be regarded as seeking a unique $T_{g}^{'} \in G^{'}$ that satisfies:
\begin{equation} \label{equivariance}
	\Phi(T_{g}(\mathcal{X}))=T_{g}^{'}(\Phi(\mathcal{X}))
\end{equation}
where $T_{g}^{'}$ is an action in a group structure $G^{'}$ and $\Phi$ denotes the feature mapping function. For brevity, it is usually written as $\Phi(T_{g}(\mathcal{X}))=T_{g}(\Phi(\mathcal{X}))$ since $T_{g}^{'}=T_{g}$ and then $ G^{'}=G$. However, we prefer the former format since $ \Phi(\mathcal{X}) $ and $ T_{g}(\mathcal{X}) $, perhaps, lie in the different domains. Two strategies can be derived from the definition to achieve the equivariance to transformations. On the one hand, $T_{g}^{'}(\Phi(\mathcal{X}))$ indicates
an explicit way to learn equivariance of transformations by transforming kernels or feature maps extracted from the input image, such as \cite{group:cohen2016group,group:dieleman2016exploiting}. However, these methods are generally inefficient because they require complicated permutations of each convolution kernel in all convolutional layers and need retraining on large-scale datasets. In addition, they neglect the
manipulation of shared weights between convolution kernels, which makes them difficult to transfer or scale to new challenging tasks. $\Phi(T_{g}(\mathcal{X}))$, on the other hand, offers an option to achieve transformation-equivariant by transforming input image directly. However, this branch arises less attention or has been referred to data augmentation method \cite{dataaugmentation}.  

To cope with the abovementioned problems, we propose
a novel framework to achieve equivariance by directly transforming input images and extracting the corresponding features with multiple CNNs. As shown in Fig \ref{idccp:idccp}, we first transform the input image according to a $ D_{4} $ transformation group that consists of image reflections and rotations by multiples of $90^{\circ}$. The main reason for choosing the $ D_{4} $ group is that the group is a regular and symmetrical polygon. In other words, it implies that any actions in a $ D_{4} $ group can prevent the image transformation from distortion. Once the transformed images have been obtained, we focus on seeking for an architecture that is effective to retain the group structure during the feature extraction. The naive way is that we adopt as many CNN networks as the order of the $ D_{4} $ group. However, this method will exponentially increase the
computational burdens. To address this problem, we exploit a Siamese-style architecture for feature extraction, which allows the weights to be shared among all subnetworks. To show how it works for preserving group structure, we provide the following proposition and the corresponding proof.           

\begin{prop}
	Let $\mathcal{X}$ be a set of images with the structure of symmetry square dihedral $ D_{4} $ group, so $ D_{4} = \left \langle r,m:r^{4}=m^{2}=e, rm=mr^{-1} \right \rangle $ and let $\Phi:Siam(\mathcal{X}) \to \mathcal{F}$ be the feature extraction function. Then, the resulting features $\mathcal{F}$ will be given in the structure of the $ D_{4} $ group.
\end{prop}

\begin{proof} 
	Let $T_{g}(\mathcal{X})$ be an action result of input $ D_{4} $ group image and $ K $ be the convolution kernel of general CNN. The convolution operation on a 2-D image can be denoted as:
	\begin{equation}\label{idccp:convolution}
		\left[ {T_{g}(\mathcal{X}) * K} \right](i,j) = \sum\limits_u {\sum\limits_v {T_{g}(\mathcal{X})(u,v)K(i - u,j - v)} },
	\end{equation}
	{Then, we can exploit $u \to u+t, v \to v+t$, $u \to -u, v \to -v$, and $ (u,v) \to r(u,v) $ (i.e., the substitution does not change the summation bounds since rotation is a symmetry of the
		sampling grid) to prove the relationships between convolution and translation, flip, and rotation, respectively. Results are:
		\begin{equation}\label{idccp:equivariant}
			\begin{split}
				&[{\Phi_t}T_{g}(\mathcal{X})] * K(i,j)={\Phi_t}[T_{g}(\mathcal{X}) * K](i,j)\\
				&[{\Phi_m}T_{g}(\mathcal{X})] * K(i,j)={\Phi_m}[T_{g}(\mathcal{X}) * \Phi_{-m}K](i,j)\\
				&[{\Phi_r}T_{g}(\mathcal{X})] * K(i,j)={\Phi_r}[T_{g}(\mathcal{X}) * {\Phi_{{r^{ - 1}}}}K](i,j)
			\end{split},
	\end{equation}}
\end{proof}
A similar visual proof of the abovementioned relationships
between convolution and transformations can be found in \cite{group:dieleman2016exploiting}. Furthermore, the pooling function that exists in CNN architecture has been proven to be commuted with the group action \cite{group:cohen2016group}. Hence, if an ordinary Siamese-style CNN learns transformed copies of the input image, the stack of feature maps will attain the same group structure as the transformed copies. It must be emphasized that the orientations of rotation may appear in either clockwise or counterclockwise depending on the implementation environment. If we let $ T_{g} $ and $ T_{g}^{'} $ be actions on the sets of $\mathcal{X}$ and $\mathcal{F}$ that satisfy $ T_{g_{1}g_{2}}= T_{g_{1}} \bullet T_{g_{2}}$ and $ T_{g_{1}g_{2}}^{'}= T_{g_{1}}^{'} \bullet T_{g_{2}}^{'}$, the transformations $ T_{g} $ and $ T_{g}^{'} $ will induce actions $ \textbf{T}_{g} $ and $ \textbf{T}_{g}^{'} $ on the space of $\mathcal{X}$ and $\mathcal{F}$. The difference between two spaces of $\mathcal{X}$ and $\mathcal{F}$ is the space field rather than the group structure. Thus, the transformation group of the input image can be preserved by using the Siamese-style CNNs. 
\begin{table*}[]
	\caption{The irreducible representations of the roto-reflection D4 group \cite{group:cohen2016steerable}.}
	\centering
	\scalebox{1.1}{
		\setlength{\tabcolsep}{3mm}{
			\begin{tabular}{c|c|c|c|c|c|c|c|c}
				\toprule
				Irrep. & e                                                     & r        & $ r^2 $ & $ r^3 $ & m        & mr       & m$ r^2 $ & m$ r^3 $ \\ \midrule
				$ \rho_{1,1} $                          & {[}1{]}                                               & {[}1{]}  & {[}1{]}              & {[}1{]}              & {[}1{]}  & {[}1{]}  & {[}1{]}               & {[}1{]}               \\
				$ \rho_{1,-1} $
				& {[}1{]}                                               & {[}1{]}  & {[}1{]}              & {[}1{]}              & {[}-1{]} & {[}-1{]} & {[}-1{]}              & {[}-1{]}              \\
				$ \rho_{-1,1} $
				& {[}1{]}                                               & {[}-1{]} & {[}1{]}              & {[}-1{]}             & {[}1{]}  & {[}-1{]} & {[}1{]}               & {[}-1{]}              \\
				$ \rho_{-1,-1} $
				& {[}1{]}                                               & {[}-1{]} & {[}1{]}              & {[}-1{]}             & {[}-1{]} & {[}1{]}  & {[}-1{]}              & {[}1{]}               \\
				$ \rho_{2} $
				& $ \begin{bmatrix} 1 & 0 \\ 0 & 1 \\ \end{bmatrix} $ &$ \begin{bmatrix} 0 & -1 \\ 1 & 0 \\ \end{bmatrix} $          &$ \begin{bmatrix} -1 & 0 \\ 0 & -1 \\ \end{bmatrix} $                      &$ \begin{bmatrix} 0 & 1 \\ -1 & 0 \\ \end{bmatrix} $                      &$ \begin{bmatrix} 1 & 0 \\ 0 & -1 \\ \end{bmatrix} $          &$ \begin{bmatrix} 0 & 1 \\ 1 & 0 \\ \end{bmatrix} $          &$ \begin{bmatrix} -1 & 0 \\ 0 & 1 \\ \end{bmatrix} $                       &$ \begin{bmatrix} 0 & -1 \\ -1 & 0 \\ \end{bmatrix} $ \\ \bottomrule                      
	\end{tabular}}}\label{idccp:d4_group}
\end{table*}
\subsection{Invariant Feature Learning Guides}
Learning invariant features, as a particular case of learning equivariant features, is essential for many recognition tasks. It turns out that adopting a Siamese-style architecture can preserve the structure of the predefined transformations of inputs $ \mathcal{X} $. The next step is to find the invariant subspace from the generated feature space $\mathcal{F}$. Because we assume that $ \rho (g) $ are all orthogonal representations, it means that they are also unitary representations that cannot be decomposed, thus enabling us to derive invariant subspaces from the perspective of irreducible representations. Taking the D4 group as an example, its irreducible representations have been summarized in TABLE. \ref{idccp:d4_group} where the orthogonality of the characters of representations can be verified.  

Considering the fact that orthogonal representation is a
real analog of unitary representation, the whole representation space can be formed by calculating the direct sum of all irreducible representations. For example, given a representation $\rho$, it can be decomposed by $\rho\simeq\lambda_{1}\tau_{1}\oplus\lambda_{2}\tau_{2}\oplus\dots\lambda_{T}\tau_{T}$. As the characteristic function of $\rho$ has been defined as $\mathcal{X}_{\rho}(g)=\text{tr}(\rho(g))$ with the matrix form $\rho(g)$ of $\rho$, the corresponding coefficients can be computed by using $\lambda_{t}=\frac{1}{|G|}\sum\limits_{g\in G} {\overline{\mathcal{X}_{\rho}(g)}\mathcal{X}_{\tau_{t}(g)}}$. The operator that projects $ \rho $ to $n_{t}\tau_{t}$ can be achieved by following $P_{\tau_{t}}=$dim$(\tau_{t})\sum\limits_{g\in G} {\overline{\mathcal{X}_{\tau_{t}}(g)}\rho(g)}$. Since $\mathcal{X}_\textbf{1}(g)=1$, we can obtain the trivial representation by calculating the average of $\rho(g)$:
\begin{equation}\label{idccp:trivial}
	P_{trivial}=\dfrac{1}{|G|}\sum\limits_{g\in G} {\rho(g)}.
\end{equation}
When we adopt the above trivial representation to train the classifier, the learned weights lie in the subspace of the entire action space (i.e., the average of all $\rho(g)$ is a subspace that is invariant to $ T $-actions). To reveal the role of learning the trivial representation, we give the following theorem.

\begin{thm}
	{Given an input sample space $\mathcal{S} = \mathcal{X} \times \mathcal{Y} =  {\left\lbrace(x_{n},y_{n}) \right\rbrace}_{n=1}^{N} \in \mathbb{R}^{d}$, which is structured by a set of orthogonal transformation group G. Then the solution of minimizing the L2 regularized convex loss function: 
		\begin{equation}
			\mathop {\min }\limits_{w,b} \frac{1}{N}\sum\limits_{n=1}^N {l\left( {{{\left\langle {{w^{\top}}{x_n} + b} \right\rangle }_{\mathbb{R}}},{y_{n}}} \right) + \frac{\lambda }{2}} ||w||^{2}
		\end{equation} lies in a vector subspace that is G-invariant, and the general error of the algorithm may be up to a factor $\sqrt{T}$ smaller than the general error of a non-invariant learning algorithm.} 
\end{thm}

\begin{proof}
	{The proof of G-invariant has been given by \cite{mukuta2019invariant} from the irreducible representation in the complex space, while Sokolic \textit{et al.} \cite{sokolic2017generalization} exploited a covering number to prove the general error of the invariant algorithm.} 
\end{proof} 
For more details, we refer readers to \cite{mukuta2019invariant} and\cite{sokolic2017generalization} and reference herein. This theorem also induces essential properties of the trivial representation. Formally, for all $g \in G$, we can have:   
\begin{equation}
	\rho(g)w=w\Leftrightarrow\rho(g)w \subseteq w \ and\ P_{trivial} w=w
\end{equation} 
The aforementioned theorem proves the G-invariance of augmented space contributes to reducing the general error of the learning algorithm but neglects to handle the massive parameters of the learning algorithm and the high-dimensional feature space. Instead, we deploy the learning algorithm to a shared-weights Siamese-style network and supply an effective compressible tensor representation in the following section.
\subsection{Compressible Covariance Pooling}\label{idccp:sec_stiefel}
\begin{table}[]
	\caption{Tensor product of irreducible representation of the roto-reflection D4 group \cite{mukuta2019invariant}.}
	\centering
	\scalebox{0.9}{
		\setlength{\tabcolsep}{1.2mm}{
			\begin{tabular}{c|c|c|c|c|c}
				\toprule
				Irrep. & $ \rho_{1,1} $  & $ \rho_{1,-1} $  &$ \rho_{-1,1} $   &$ \rho_{-1,-1} $   &$ \rho_{2} $   \\ \midrule
				$ \rho_{1,1} $&$ \rho_{1,1} $   &$ \rho_{1,-1}$   &$ \rho_{-1,1} $  &$ \rho_{-1,-1} $  &$ \rho_{2} $  \\
				$ \rho_{1,-1} $&$ \rho_{1,-1}$   &$ \rho_{1,1} $   &$ \rho_{-1,-1} $  &$ \rho_{-1,1} $  &$ \rho_{2} $  \\
				$ \rho_{-1,1} $&$ \rho_{-1,-1} $   &$ \rho_{-1,1} $   &$ \rho_{1,-1} $  &$ \rho_{1,1} $  &$ \rho_{2} $  \\
				$ \rho_{-1,-1} $&$ \rho_{-1,-1} $   &$ \rho_{-1,1} $   &$ \rho_{1,-1} $  &$ \rho_{1,1} $  &$ \rho_{2} $  \\
				$ \rho_{2} $&$ \rho_{2} $   &$ \rho_{2} $   &$ \rho_{2} $  &$ \rho_{2} $  &$ \rho_{1,1} $$\bigoplus$$ \rho_{1,-1} $$\bigoplus$ $ \rho_{-1,1} $$\bigoplus$$ \rho_{-1,-1} $ \\\bottomrule
	\end{tabular}}}\label{idccp:tensor_d4group}
\end{table}
Covariance pooling, as a form of the second-order statistics feature, aims to establish the correlation between the spatial and channels of local CNN features to aggregate more distinguishing information. Suggested by \cite{sos:li2017second,sos:acharya2018covariance}, and \cite{sos:li2018towards}, we perform the second-order pooling in the form of a scatter covariance matrix:
\begin{equation}\label{idccp:covariance}
	\boldsymbol{\Sigma}=\dfrac{1}{\text{hw}}\sum_{i=1}^{\text{hw}}P\left((\textbf{f}_{i}-\overline{\textbf{f}})(\textbf{f}_{i}-\overline{\textbf{f}})^{\top}\right)=\dfrac{1}{\text{hw}}P_{trivial}\left(\textbf{F} \overline{\textbf{I}}\textbf{F}^{\top}\right).
\end{equation}  
where \textit{w} and \textit{h} are feature width and height. $P_{trivial}$ is projection function that we have introduced before. $\overline{\textbf{f}}=\dfrac{1}{\text{hw}}\sum_{i=1}^{\text{hw}}\textbf{f}_{i}$ is the mean of feature vectors. $ \overline{\textbf{I}}=\textbf{I}-\dfrac{1}{\text{hw}}\textbf{1}\textbf{1}^{\top} \in \mathbb{R}^{\text{hw}\times \text{hw}}$ is the centering matrix, where $ \textbf{I} $ and $ \textbf{1} $ denote the identity matrix and the all-ones matrix, respectively. 

Since the projection function $P_{trivial}$ is employed in the tensor space, the tensor product representation needs to be given concerning the irreducible representation in the D4 group. According to the distributive property of tensor product representation (e.g., given two representations $\rho$ and $\sigma$, it satisfies $(\rho_{1}\oplus\rho_{2})\otimes(\rho_{3}\oplus\rho_{4})=(\rho_{1}\otimes\rho_{3})\oplus(\rho_{2}\otimes\rho_{3})\oplus(\rho_{1}\otimes\rho_{4})\oplus(\rho_{2}\otimes\rho_{4})$ and $\mathcal{X}_{\rho\otimes\sigma}(g)=\mathcal{X}_{\rho}(g)\mathcal{X}_{\sigma}(g)$, we can calculate the
tensor product representations of irreducible representations. Combining the fact that the tensor product of irreducible representation and 1-D representation is irreducible, it allows
us to decompose tensor products of D4 group and present
the results in TABLE \ref{idccp:tensor_d4group}. For verifying the results, we take two representations $\rho(e)$ and $\rho(m)$ in TABLE \ref{idccp:d4_group} as an example, and the corresponding tensor product representations become 4-D vectors such that $\rho(e)=\begin{bmatrix} 1 & 0 & 0& 0\\ 0 & 1& 0& 0 \\ 0 & 0& 1& 0\\ 0& 0& 0& 1\\\end{bmatrix}$ and $\rho(m)= \begin{bmatrix} 1 & 0 & 0& 0\\0&1&0&0\\ 0& 0& -1& 0 \\ 0& 0& 0& -1 \\ \end{bmatrix} $, respectively.

The obtained covariance matrix can be regarded as a form of representation, which is capable of capturing more information than the ordinary first-order statistical feature. However, its shortcomings are also obvious. The first and foremost drawback of such covariance pooling is its high dimensionality. Taking VGG architecture \cite{vgg} as an example, the dimension of the vectorized covariance matrix generated from the last convolution layer will be $ 2^{18} $. Rank deficiency is another weakness of covariance pooling because the number of CNN channels is much larger than the product of feature height and width.

The abovementioned reasons promote us to discover a
compact form of covariance pooling. Considering that the
covariance matrix is an SPD matrix, it is necessary to retain the geometry of the SPD manifold while reducing the matrix dimension. To accomplish this goal, we provide a method based on the following proposition.

\begin{prop}
	{Let $ \boldsymbol{\Sigma} \in \mathbb{R}^{d \times d}$ be the covariance matrix generated from the last convolution layer and $ \textbf{W} \in \mathbb{R}^{d \times \hat{d}} $ be an orthogonal, row full rank matrix with $ \hat{d}<d $. Then, the bilinear form of transformation matrix $ \textbf{W}$ maps $\boldsymbol{\Sigma}$ to a valid SPD matrix $ \boldsymbol{\hat{\Sigma}} \in \mathbb{R}^{\hat{d} \times \hat{d}} $.}
\end{prop}  

\begin{proof}
	The bilinear mapping function can be generally denoted as $ \mathcal{B}: \boldsymbol{\Sigma} \times \textbf{W} \to \boldsymbol{\hat{\Sigma}} $. In order to express it more accurately, we can rewrite it in the form of: $ \boldsymbol{\hat{\Sigma}}= \textbf{W}^{\top}\boldsymbol{\Sigma}\textbf{W} $. Due to the orthogonality and row full rank of transformation matrix $ \textbf{W} $, the elements generated by transformation weights are naturally located on a non-compact Stiefel manifold $ \mathcal{S}^{*}(\hat{d},d)\triangleq \left\lbrace  \textbf{W} \in \mathbb{R}^{d \times \hat{d}}: \textbf{W}^{\top}\textbf{W}=\textbf{I}_{\hat{d}}\right\rbrace  $ and can be transformed into a compact manifold $ \mathcal{S}(\hat{d},d) $. Then, the resulting matrix  $ \boldsymbol{\hat{\Sigma}} \in \mathbb{R}^{\hat{d} \times \hat{d}} $ is a valid but very compact SPD matrix because $ \hat{d}<d $.
\end{proof} 

The abovementioned claim and proof are trivial, but it
guides us to convert those high-dimensional SPD matrices $ \boldsymbol{\Sigma} $ {to} new, low-dimensional SPD matrices $ \boldsymbol{\hat{\Sigma}}  $ with $ \hat{d}<d$, $\boldsymbol{\hat{\Sigma}}\in Sym_{\hat{d}}^{+}$. Compared with most existing methods that directly map SPD manifold into the Euclidean space \cite{sos:lin2015bilinear,sos:lin2017improved,sos:kong2017low,sos:gao2016compact,sos:li2017second,sos:li2018towards}, our method can certainly preserve the inherent manifold structure of high-dimensional SPD matrices. However, given a non-compact Stiefel manifold, a matrix form of writing linearly independent column vectors (i.e., $ \hat{d} $-frames), has no closed-form of geodesic curves. In other words, it is infeasible to optimize on the manifold directly \cite{sos:fiori2010learning}. The
relatively tractable strategy is to endow non-compact Stiefel manifold with a pseudo-Riemannian manifold so that the gradient of geodesic distance can be derived from a smooth manifold and present in a closed form. To achieve this target, we impose the orthogonal constraints on $ \textbf{W} $ (precisely speaking, it is semi-orthogonal matrix under this scenario). Consequently, the entities of transformation weight $ \textbf{W} $ reside on a compact Stiefel manifold $ \mathcal{S}(\hat{d},d) $, which allows us to find the optimal solutions of the weight matrix.    

Furthermore, the abovementioned function for feature dimension reduction can also be regarded as an intertwining operator when we impose orthogonal constraints $ \textbf{W}^{\top}\textbf{W}=\textbf{I}_{\hat{d}} $ on transformation weight. Recalling the introduction of
intertwining in preliminaries, the produced projection space is also the representation space. Thus, the low-dimensional representation can be achieved by imposing low-rank constraints on weight $ \textbf{W} $. Specifically, we can first line up the eigenvalues of $ \boldsymbol{\Sigma} $ by employing eigenvalue decomposition function and then find the elements with the larger variance to retain. However, matrix decomposition often requires more computational costs and time-consuming \cite{sos:li2017second,mg_cap}. Rather than using cumbersome decomposition functions, the bilinear mapping function can transform the input SPD matrix into a new, low-dimensional SPD matrix that is useful for subsequent optimization.                 
\subsection{Invariant Classifier Training}
The compressible covariance pooling method has been
described in the last section, which maps the high-dimensional manifold to a low-dimensional compact manifold. Different from the mainstream methods, our algorithm deduces a rank efficient representation on manifold space while retaining the inherent manifold structure.   

The elements of the resulting low-dimensional SPD matrices reside on the Riemannian manifold, which needs to be transformed into the Euclidean space so that the distance between different elements can be measured by the Euclidean operations. The natural choice is to employ the logarithm of SPD matrices since it reflects the true geodesic distance of the manifold. Furthermore, the logarithm of an SPD matrix will give rise to the matrix with a Lie group, and then, all Euclidean operations can be adopted. However, the logarithm will change the magnitude order of small eigenvalues and usually not
robust in practice \cite{sos:li2017second,sos:lin2017improved}. Instead, we are committed to learning more robust square root normalization of matrices, which can be considered as the approximate Riemannian geometry in covariance matrices \cite{sos:li2017second}.  

It is well-known that any SPD matrix has a unique square
root, which can be obtained by using SVD or EIG. Although
SVD or EIG yield the accurate solution of the square root
of a matrix, they are time-consuming and often cannot be
well-supported by GPU acceleration \cite{sos:li2017second,mg_cap}. Inspired by \cite{sos:li2018towards}, we adopt the iSQRT-COV approach that uses a variation of the Newton method to iteratively calculate the square root of the matrix. Especially, given $\textbf{C}_{0}= \frac{\boldsymbol{\hat{\Sigma}}}{\text{tr}(\boldsymbol{\hat{\Sigma}})}$ and $\textbf{D}_{0}=\textbf{I}$, the Newton-Schulz method \cite{sos:li2018towards} allows us to compute the square root $\textbf{C}$ of $\boldsymbol{\hat{\Sigma}}$ by using the following iterations:
\begin{equation}\label{idccp:Newton_Schulz}
	\begin{split}
		&\textbf{C}_{j}=\frac{1}{2}\textbf{C}_{j-1}(3\textbf{I}-\textbf{D}_{j}\textbf{C}_{j-1}),\\
		&\textbf{D}_{j}=\frac{1}{2}(3\textbf{I}-\textbf{D}_{j}\textbf{C}_{j-1})\textbf{D}_{j},
	\end{split}
\end{equation}
where $j=1,\dots,J$ is the iteration steps. With the condition of $||\boldsymbol{\hat{\Sigma}}-\textbf{I}||<\frac{1}{2}$, $\textbf{C}_{j}$ and $\textbf{D}_{j}$ are guaranteed to quadratically converge to $\textbf{C}^{\frac{1}{2}}$ and $\textbf{C}^{-\frac{1}{2}}$, respectively. Briefly, it means that $\textbf{C}^{2}=\boldsymbol{\hat{\Sigma}}$ and $\textbf{C}=\boldsymbol{\Psi}\boldsymbol{\Lambda}\boldsymbol{\Psi}^{\top}$ described in EIG format, where $\textbf{C}=\boldsymbol{\Psi}$ is an orthogonal matrix and $\boldsymbol{\Lambda}=(\boldsymbol{\lambda}_{1}^{\frac{1}{2}},\boldsymbol{\lambda}_{2}^{\frac{1}{2}},\dots,\boldsymbol{\lambda}_{d^{'}}^{\frac{1}{2}})$ is a diagonal matrix.        

Once the square-root of the SPD matrix is obtained, we can
adopt the Euclidean operations to measure the distance of
elements on the flatted Stiefel Manifold. Considering the fact that the initialization of $\textbf{C}_{0}$ has changed the magnitude of the matrix value, we then use $\hat{\textbf{C}}=\sqrt{\text{tr}(\boldsymbol{\hat{\Sigma}})}\textbf{C}_{j}$ to counteract such changes \cite{sos:li2018towards}. The resulting matrix $\hat{\textbf{C}}$ can be used to train the classifier. Let us suppose that $ \hat{\textbf{W}} $ be the corresponding weight matrix of $\hat{\textbf{C}}$. The objective function in $ \textbf{Theorem 1.} $ can be rewritten by substituting \textit{w} with $ \hat{\textbf{W}} $, and then yield the following expression:
\begin{equation}\label{idccp:obj_func}
	\begin{split}
		&\mathop {\min }\limits_{\hat{\textbf{W}},b} \frac{1}{N}\sum\limits_{i=1}^N l \left( {{{\left\langle {\text{tr}\left( {\hat{\textbf{W}}^{\top}\hat{\textbf{C}}} \right) + b} \right\rangle }_{\mathbb{R}}},{y_{i}}} \right) + \frac{\lambda }{2}||\hat{\textbf{W}}||^{2}\\
		=&\mathop {\min }\limits_{\hat{\textbf{W}},b} \frac{1}{N}\sum\limits_{i=1}^N l \left( {{{\left\langle {\text{tr}\left(\hat{\textbf{W}}^{\textbf{(1)}}\hat{\textbf{C}}^{\textbf{(1)}}\right)+ b} \right\rangle }_{\mathbb{R}}}},{y_{i}} \right)+ \frac{\lambda }{2}||\hat{\textbf{W}}||^{2}.\\
	\end{split}
\end{equation}
For brevity, we omit the transpose operator $\top$ in the last line since $\hat{\textbf{W}}=\hat{\textbf{W}}^{\top}$. The final result highlights the key advantage of our classifier, which avoids the direct optimization on the original high-dimensional weights $\textbf{W}$.
\subsection{Back-propagation} 
Stochastic gradient descent (SGD), as one of the most
popular gradient calculation algorithms, is widely adopted
for training deep CNNs. In our scenario, we employ SGD
to compute the gradient of the given objective function with respect to the transformation matrix $\textbf{W}$ and the second-order statistical feature $\boldsymbol{\Sigma}$. Let us write the derivative of $\hat{\textbf{C}}$ as $\left( {\frac{{\partial l}}{{\partial \hat{\textbf{C}}}}} \right) $ that derives from the Softmax layer. Then, we can use the following chain rules to calculate the matrix derivatives:  
\begin{equation}\label{idccp:chain_rule}
	\begin{split}
		&\text{tr}\left( {{{\left( {\frac{{\partial l}}{{\partial \hat{\textbf{C}}}}} \right)}^{\top}}d\hat{\textbf{C}}} \right) = \text{tr}\left( {{{\left( {\frac{{\partial l}}{{\partial {\hat{\textbf{C}}_{J}}}}} \right)}^{\top}}d{\hat{\textbf{C}}_{J}} + {{\left( {\frac{{\partial l}}{{\partial \boldsymbol{\hat{\Sigma}} }}} \right)}^{\top}}d\boldsymbol{\hat{\Sigma}} } \right),\\
		&\text{tr}\left( {{{\left( {\frac{{\partial l}}{{\partial \textbf{W}}}} \right)}^{\top}}d\boldsymbol{\Sigma}} \right) = \text{tr}\left( {{{\left( {\frac{{\partial l}}{{\partial \boldsymbol{\hat{\Sigma}}}}} \right)}^{\top}}d{\boldsymbol{\hat{\Sigma}}} + {{\left( {\frac{{\partial l}}{{\partial \textbf{W}}}} \right)}^{\top}}d\boldsymbol{\Sigma} } \right),\\
	\end{split}
\end{equation}
where $d\hat{\textbf{C}}$ is the variation of $\hat{\textbf{C}}$. According to expression at the first line, we can derive the derivative of $\boldsymbol{\hat{\Sigma}}$ through some manipulations. For more details, we refer readers to \cite{sos:li2018towards} and reference it herein. Once we obtained $\frac{{\partial l}}{{\partial \boldsymbol{\hat{\Sigma}}}}$, we can exploit it to compute the gradient for updating $\textbf{W}$.    

As described in \ref{idccp:sec_stiefel}, we project all elements on the Stiefel manifold $ \mathcal{S}(\hat{d},d) $ into the Euclidean space so that we can use the Euclidean operations to measure the distance between projected elements. However, directly using the back propagation rules in the Euclidean space to calculate the gradient of the Stiefel manifold cannot guarantee that the orthogonality of weights $\textbf{W}$. To this end, we introduce the Euclidean inner product in the tangent space of the Stiefel manifold as a new strategy for updating the gradient of our covariance pooling. Therefore, the Stiefel manifold is transformed into a Riemannian manifold so that we can borrow the method of optimizing the Riemannian manifold to calculate the gradient of the Stiefel manifold. When employing the Euclidean inner product, the corresponding gradient of the current points $\textbf{W}^{t}$ on the Riemannian manifold $G_{e}^{\mathcal{S}}l(\textbf{W}^{t})$ can be obtained by:
\begin{equation}\label{idccp:riemannian_grad}
	G_{e}^{\mathcal{S}}l(\textbf{W}^{t}) = \frac{{\partial l}}{{\partial \textbf{W}^{t}}}-\textbf{W}^{t}\left(\frac{{\partial l}}{{\partial \textbf{W}^{t}}}\right)^{\top}\textbf{W}^{t},
\end{equation}
where $\frac{{\partial l}}{{\partial \textbf{W}^{t}}}$ is the normal component of the gradient in the Euclidean space, which can be obtained by using the second expression of Eq.(\ref{idccp:chain_rule}) as:
\begin{equation}\label{euclidean_grad}
	\frac{{\partial l}}{{\partial \textbf{W}^{t}}}=2{\frac{{\partial l}}{{\partial \boldsymbol{\hat{\Sigma}}}}}\textbf{W}^{t}\boldsymbol{\Sigma},
\end{equation}
When we obtain the Riemannian gradient, we need to seek the descent direction of the gradient (i.e., we use steepest gradient descent in our scenario) and ensure that the new update points $\textbf{W}^{t+1}$ are located on the Stiefel manifold. To achieve this, we adopt the QR-decomposition retraction $ Z_{\textbf{W}}(\xi)= qf(\textbf{W}+\xi)$ introduced in \cite{sos:huang2017riemannian,edelman1998geometry,absil2009optimization}. Here, $qf(\cdot)$ is the adjusted Q factors of the QR-decomposition and R factors in an upper triangular matrix with strictly positive elements on the diagonal. Thus, the decomposition is guaranteed to be unique and orthogonal. Through defining the learning rate as $\eta$, we can compute the new point by:
\begin{equation}
	\textbf{W}_{t+1}=qf\left(\textbf{W}_{t}-\eta G_{e}^{\mathcal{S}}l(\textbf{W}^{t}) \right),
\end{equation}
Once the derivation of $\frac{{\partial l}}{{\partial \boldsymbol{\Sigma}}}$ has been achieved, we can derive the derivative for the input feature $\textbf{F}$ with using:
\begin{equation}\label{feat_derivation}
	\frac{{\partial l}}{{\partial \textbf{F}}}=\left( \frac{{\partial l}}{{\partial \boldsymbol{\Sigma}}}\left( \frac{{\partial l}}{{\partial \boldsymbol{\Sigma}}}\right)^{\top}\right)\hat{\textbf{I}}\textbf{F}.   
\end{equation}
\section{Experiments}
\subsection{Datasets Description}
To evaluate the effectiveness of the proposed method,
we carried out comprehensive experiments on four publicly available aerial image datasets. Experimental data sets
include UC Merced Land-Use dataset \cite{dataset:yang2010bag}, Aerial Image dataset \cite{dataset:xia2017aid}, NWPU-RESISC45 dataset \cite{dataset:cheng2017remote}, and a recently released OPTIMAL-31 datasets \cite{dataset:wang2018scene}. The statistics of datasets have been summarized in TABLE \ref{idccp:datasets_sum}. In addition to the information listed in TABLE \ref{idccp:datasets_sum}, the spatial resolution of aerial images is another important factor that makes aerial images differ from ordinary scene images. Especially, the spatial resolution of the UC Merced Land-Use dataset \cite{dataset:yang2010bag} is about 0.3m/pixel, while it becomes 0.5-8m/pixel on AID \cite{dataset:xia2017aid}, and it is even more diversified on NWPU-RESISC45 dataset  \cite{dataset:cheng2017remote}, which is about 0.2-30m/pixel. Through comparing with the information described in Section \ref{introduction}, we can see that most of the data sets used to evaluate our model can satisfy the definition of high spatial resolution images. 
\begin{table}[]
	\caption{The statistics of remote sensing scene datasets.}
	\scalebox{0.88}{
		\begin{tabular}{c|c|c|c|c}
			\toprule
			Datasets      & \ No. Images & \ No. Class & \begin{tabular}[]{@{}c@{}}No. Images\\ (Per-class)\end{tabular}\    & \ Image Size \\ \midrule
			\begin{tabular}[]{@{}c@{}}UC Merced\\ Land-Use\end{tabular}  \cite{dataset:yang2010bag}     & 2,100         & 21           & 100            		& 256$\times$256		           \\
			AID \cite{dataset:xia2017aid}           & 10,000        & 30           & 220$\sim$420   		& 600$\times$600		              \\
			NWPU-RESISC45 \cite{dataset:cheng2017remote} & 31,500        & 45           & 700          		& 256$\times$256		                \\
			OPTIMAL-31 \cite{dataset:wang2018scene}     & 1,861         & 31           & 60            		& 256$\times$256		           \\ \bottomrule   
	\end{tabular}}\label{idccp:datasets_sum}
\end{table}
\subsection{Implement Details}
\begin{table*}[]
	\centering
	\caption{Comparison with the Bilinear pooling method \cite{sos:lin2015bilinear} in terms of feature dimensionality, computational complexity and the number of parameters ( ResNet50 \cite{resnet}-based Siamese-style architecture ). Where $d_{p}=512$ and $K$ denote the projection layer and the number of categories, respectively. (w/ and w/o indicate with projection layer and without projection layer, respectively.)}
	
	\begin{tabular}{c|c|cc|cc|c}\toprule
		Methods      &Feature Dim. & Feature Comp. & Classifier Comp. & Feature Param. &Classifier Param.&Model Param.\\ \midrule
		Bilinear pooling \cite{sos:lin2015bilinear} (w/o $ d_{p} $)    & $d^{2}$\ [4,194K]         & $O(hwd^{2})$           & $O(Kd^{2})$              & 0             & $Kd^{2}$\ [$K\cdot$16MB]        &[118MB]\\ \midrule
		Our IDCCP-512 (w/ $ d_{p} $)          & $\hat{d}_{1}^{2}$\ [256K]        & $O(hwd_{p}d+hw\hat{d}_{1}^{2})$           & $O(Kd_{1}^{2})$        &$dd_{p}$\ [4MB] &$Kd_{1}^{2}$\ [$K\cdot$1MB]              &[30MB]\\
		Our IDCCP-64 (w/ $ d_{p} $)           & $\hat{d}_{2}^{2}$\ [4K]        & $O(hwd_{p}d+hw\hat{d}_{2}^{2})$           & $O(Kd_{2}^{2})$          &$dd_{p}$\ [4MB]  &$Kd_{2}^{2}$\ [$K\cdot$16KB]             &[25MB]\\
		\bottomrule
	\end{tabular}\label{idccp:compare_complexity}
\end{table*}
\begin{table*}[]
	\centering
	\caption{Comparison with state-of-the-art approaches in terms of overall accuracy and standard deviation (\%). T.R. is the abbreviation of Training Ratio.}
	\begin{tabular}{cc|cc|cc|c|c}
		\toprule
		\multirow{2}{*}{}                                            & \multirow{2}{*}{Methods}  &  \multicolumn{2}{c}{NWPU-RESISC45 \cite{dataset:cheng2017remote}}         & \multicolumn{2}{c}{AID \cite{dataset:xia2017aid}}                   & {UC-Merced Land-Use \cite{dataset:yang2010bag}}     & \multirow{2}{*}{Publication Year}\\\cline{3-7} 
		&                         & T.R.=10\% & T.R.=20\% & T.R.=20\% & T.R.=50\%& T.R.=80\% &\\ \midrule
		
		& AlexNet+SVM \cite{dcnn}             & 81.22$ \pm $0.19          & 85.16$ \pm $0.18          & 84.23$ \pm $0.10          & 93.51$ \pm $0.10                      & 94.42$ \pm $0.10          &2018\\
		& GoogLeNet+SVM \cite{dcnn}           & 82.57$ \pm $0.12          & 86.02$ \pm $0.18          & 87.51$ \pm $0.11          & 95.27$ \pm $0.10                       & 96.82$ \pm $0.20          &2018\\
		& VGGNet+SVM \cite{dcnn}               & 87.15$ \pm $0.45          & 90.36$ \pm $0.18          & 89.33$ \pm $0.23          & 96.04$ \pm $0.13                       & 97.14$ \pm $0.10          &2018\\
		
		& MSCP with AlexNet \cite{mgcap:he2018remote}         & 81.70$ \pm $0.23          & 85.58$ \pm $0.16          & 88.99$ \pm $0.38          & 92.36$ \pm $0.21                        & 97.29$ \pm $0.63          &2018\\
		& MSCP+MRA with AlexNet \cite{mgcap:he2018remote}     & 88.31$ \pm $0.23          & 87.05$ \pm $0.23          & 90.65$ \pm $0.19          & 94.11$ \pm $0.15                       & 97.32$ \pm $0.52          &2018\\
		
		& MSCP with VGGNet \cite{mgcap:he2018remote}         & 85.33$ \pm $0.17          & 88.93$ \pm $0.14          & 91.52$ \pm $0.21          & 94.42$ \pm $0.17                        & 98.36$ \pm $0.58          &2018\\
		& MSCP+MRA with VGGNet \cite{mgcap:he2018remote}     & 88.07$ \pm $0.18          & 90.81$ \pm $0.13          & 92.21$ \pm $0.17          & 96.56$ \pm $0.18                       & 98.40$ \pm $0.34          &2018\\
		& DCNN with AlexNet \cite{dcnn}     & 85.56$ \pm $0.20          & 87.24$ \pm $0.12          & 85.62$ \pm $0.10          & 94.47$ \pm $0.10                          & 96.67+0.10          &2018\\
		& DCNN with GoogLeNet \cite{dcnn}     & 86.89$ \pm $0.10          & 90.49$ \pm $0.15          & 88.79$ \pm $0.10          & 96.22$ \pm $0.10                        & 97.07+0.12          &2018\\
		& DCNN with VGGNet \cite{dcnn}         & 89.22$ \pm $0.50          & 91.89$ \pm $0.22          & 90.82$ \pm $0.16          & 96.89$ \pm $0.10                       & 98.93$ \pm $0.10                      &2018\\ 
		& RTN with VGGNet \cite{rtn}         & 89.90          & 92.71          & 92.44          & -                      & 98.96                      &2018\\
		& Two-Stream Fusion \cite{lgfusion:yu2018two}         & 80.22$ \pm $0.22          & 83.16$ \pm $0.18          & 92.32$ \pm $0.41          & 94.58$ \pm $0.25                       & 98.02$ \pm $1.03                      &2018\\
		& GCFs+LOFs \cite{lgfusion:zeng2018improving}         & -          &-           &92.48$ \pm $0.38           &96.85$ \pm $0.23                        &99.00$ \pm $0.35 &2018\\
		& CapsNet with VGGNet \cite{capsnet}         
		& 85.08$ \pm $0.13                      &89.18 $\pm$ 0.14           
		&91.63$ \pm $0.19                       &94.74$ \pm $0.17                        
		& 98.81$ \pm $0.22                      &2019\\
		& MG-CAP with Bilinear VGGNet \cite{mg_cap}         
		& 89.42$ \pm $0.19                      &91.72 $\pm$ 0.16           
		&92.11$ \pm $0.15                       &95.14$ \pm $0.12                        &98.60$ \pm $0.26 &2020\\
		& MG-CAP with Log-E VGGNet \cite{mg_cap}         
		& 88.35$ \pm $0.23                      &90.94 $\pm$ 0.20           
		&90.17$ \pm $0.19                       &94.85$ \pm $0.16                        &98.45$ \pm $0.12 &2020\\
		& MG-CAP with Sqrt-E VGGNet \cite{mg_cap}         
		& 90.83$ \pm $0.12                      &92.95 $\pm$ 0.13           
		&93.34$ \pm $0.18                       &96.12$ \pm $0.12                        &99.00$ \pm $0.10 &2020\\
		\midrule
		&IDCCP with VGGNet-512(ours)                                                     &90.88$\pm$0.18                     &92.80$\pm$0.10                     &93.58$\pm$0.24                     &96.33$\pm$0.12                       &98.45$\pm$0.12                     &-\\
		&IDCCP with VGGNet-64(ours)                                                     &89.61$\pm$0.19                    &91.75$\pm$0.18                     &92.33$\pm$0.25                     &94.82$\pm$0.22                       &97.35$\pm$0.20                     &-\\
		&\bf IDCCP with ResNet50-512(ours)                                                   &\textbf{91.55}$\pm$0.16               &\textbf{93.76}$\pm$0.12              &\textbf{94.80}$\pm$0.18               &\textbf{96.95}$\pm$0.13             &\textbf{99.05}$\pm$0.20       &-\\
		&IDCCP with ResNet50-64(ours)                                                     &91.31$\pm$0.22                     &93.66$\pm$0.21                     &94.64$\pm$0.23                     &96.73$\pm$0.18                          &98.57$\pm$0.24&-\\
		\bottomrule  
	\end{tabular}\label{idccp:overall_std}
\end{table*}
\begin{table}[]
	\centering
	\caption{Comparison with state-of-the-art methods in terms of overall accuracy and standard deviation (\%).}
	\begin{tabular}{c|c}
		\toprule
		\multirow{2}{*}{Method} & OPTIMAL-31 \cite{dataset:wang2018scene} \\\cline{2-2}
		& Training Ratio=80\% \\ \midrule
		Fine-tuning AlexNet \cite{dataset:wang2018scene}            & 81.22 $\pm$ 0.19        \\
		Fine-tuning GoogLeNet \cite{dataset:wang2018scene}          & 82.57 $\pm$ 0.12        \\
		Fine-tuning VGGNet16 \cite{dataset:wang2018scene}           & 87.45 $\pm$ 0.45        \\
		ARCNet-Alexnet \cite{dataset:wang2018scene}                 & 85.75 $\pm$ 0.35        \\
		ARCNet-ResNet34 \cite{dataset:wang2018scene}                & 91.28 $\pm$ 0.45        \\
		ARCNet-VGGNet16 \cite{dataset:wang2018scene}                & 92.70 $\pm$ 0.35        \\\midrule
		IDCCP with VGG-512 (ours)      & 93.82$\pm$0.32          \\
		IDCCP with VGG-64 (ours)      & 92.13$\pm$0.38          \\
		\textbf{IDCCP with ResNet50-512 (ours)} & \textbf{94.89}$\pm$0.22         \\
		IDCCP with ResNet50-64 (ours)      & 94.54$\pm$0.28          \\\bottomrule
	\end{tabular}\label{idccp:opt_acc}
\end{table}
We implement our method using the GPU version of Tensorflow in v1.10.0. We construct two Siamese-style architectures based on two standard CNNs, including VGGNet \cite{vgg} and ResNet50 \cite{resnet}. We remove all fully-connected layers from the original backbone networks and then insert our projection layer and compressible covariance pooling layer at the same place to train the invariant classifier. The batch size is set to 32 during training. We employ SGD with a momentum of 0.9 and a weight decay of 0.0005 to optimize the gradient. The initial learning rate is set to 0.1 and becomes 0.01 when fine-tuning the entire network. We employ exponential decay in the training process, with a decay factor of 0.9 in every 10 epochs. We adopt five-fold cross-validation to reduce the influence of the randomness and obtain reliable results. When we train our model on the UC Merced Land-Use dataset \cite{dataset:yang2010bag}, the NWPU-RESISC45 dataset \cite{dataset:cheng2017remote}, and the OPTIMAL-31 dataset \cite{dataset:wang2018scene}, we randomly crop patches of 224 $\times$ 224 pixels from the input image and flip them horizontally or vertically. During the test, the manipulation of central cropping is adopted to
obtain patches of the same size as in training. These operations are also applied to AID \cite{dataset:xia2017aid}, but the size of patches becomes 448 $\times$ 448 pixels.
\subsection{Analysis of Model Complexity}
In view of the success of bilinear pooling \cite{sos:lin2015bilinear} and its relevance to our method, we compared the differences of two
models in various aspects and listed the results in TABLE \ref{idccp:compare_complexity}. Especially, the aspects of comparison include input feature dimension, complexity and corresponding parameter size, classifier complexity and its parameter size, and overall model parameters. In order to show the function of the projection
layer, all results are obtained by employing the Siamese-style architecture based on ResNet50 \cite{resnet}. With using the projection layer, the feature dimension can be reduced to the same par with the last convolution layer in VGGNet \cite{vgg}. As shown in TABLE \ref{idccp:compare_complexity}, our invariant deep compressible covariance pooling (IDCCP) model requires an additional 4-MB feature parameter compared to the bilinear pooling model \cite{sos:lin2015bilinear}. However, this operation is more conducive to reducing the
feature dimension and, thus, greatly reducing the number of classifier parameters. Namely, our IDCCP model not only
learns compressible feature representations but also enables us to train more compact classifiers.
\subsection{Comparison with State-of-the-Arts}
We provide four variants of the IDCCP model and display
the classification results in TABLE \ref{idccp:overall_std}.  It is plain to see that
our IDCCP models achieved extremely competitive results on
all experimental datasets. In particular, the performance of the IDCCP model based on ResNet50 \cite{resnet} is superior to the latest MG-CAP model \cite{mg_cap} on all datasets and even far exceeds baseline methods (e.g., our method is improved by about 10\% compared with the standard method of AlexNet + SVM on the challenging NWPU-RESISC45 dataset \cite{dataset:cheng2017remote}). When using VGGNet \cite{vgg}, the MG-CAP model \cite{mg_cap} shows strong competitiveness in classification accuracy, but even if GPU acceleration is enabled, it requires 4.5 times the number of transformations and nearly seven times in terms of inference time. When we employ ResNet50-based Siamese-style architecture, the proposed IDCCP models can obtain accuracy rates higher than 91\% and 93\% under two split ratios on the NWPU-RESISC45 dataset \cite{dataset:cheng2017remote}, respectively. On AID \cite{dataset:xia2017aid}, we can obtain 94.80$\pm$0.18 with using 20\% training samples, which exceeds the best results of MG-CAP model \cite{mg_cap} and DCNN model \cite{dcnn} by 1.46\% and 3.98\%, respectively. Under the 50\% training ratio, the GCFs+LOFs model \cite{lgfusion:zeng2018improving} presents surprisingly better than most existing methods but still below the optimal level of our IDCCP framework. Similarly, on UC Merced Land-Use dataset \cite{dataset:yang2010bag}, our IDCCP achieves the highest accuracy among the listed algorithms. In addition, we show the comparisons of our IDCCP model with previous methods on the OPTIMAL-31 dataset \cite{dataset:wang2018scene}. As shown in TABLE \ref{idccp:opt_acc}, three variants of our proposed method achieve higher results than the latest algorithm. Even our worst IDCCP model can still exceed the result of fine-tuning AlexNet by more than
10\%. By using ResNet50 architecture, our IDCCP model can improve the optimal performance of ARCNet-VGGNet16 by 1.84\%. These indicate that the classification performance can be improved by incorporating the prior knowledge of the input image. 

Generalization ability is vitally important for measuring the effectiveness of deep learning models. By analyzing the data listed in Table \ref{idccp:overall_std}, it is not difficult to see that the variants of our IDCCP model can always bring relatively stable benefits to different data sets. Concretely, using different proportions of training data on the NWPU-RESISC45 \cite{dataset:cheng2017remote} (i.e., 10\% versus 20\% training ratios), the difference between our IDCCP model is about 2\%, but this gap is significantly enlarged on other models (e.g., about 4\% by CapsNet with VGGNet \cite{capsnet} and about 3\% by MSCP with AlexNet or VGGNet \cite{mgcap:he2018remote}). A similar degree of gain is also reflected in the AID \cite{dataset:xia2017aid} with different partitions. However, most of the existing methods are not stable enough under different partitions, including DCNN \cite{dcnn} (about 6\%-9\%), GCFs+LOFs \cite{lgfusion:zeng2018improving} (about 4\%), and SVM-based methods \cite{dcnn} (about 7\%-9\%).  It is worth noting that the actual number of samples corresponding to different training ratios on two different datasets (10\% and 20\% on NWPU-RESISC45 \cite{dataset:cheng2017remote} versus 20\% and 50\% on AID \cite{dataset:xia2017aid}) is in the same order of magnitude (3,150 on NWPU-RESISC45 \cite{dataset:cheng2017remote} versus 3,000 on AID \cite{dataset:xia2017aid}). Therefore, similar
gains in different scenarios also reflect that the robustness of our IDCCP model.

Through comparing the variants of our IDCCP model,
we found that the IDCCP model based on VGGNet \cite{vgg}  has achieved very competitive results on all experimental datasets. Especially, using VGGNet \cite{vgg}, our IDCCP model can obtain comparable results to the similar methods, such as RTN \cite{rtn} and MG-CAP \cite{mg_cap}, and even significantly better than MSCP \cite{mgcap:he2018remote}. The full-rank IDCCP model (i.e., VGGNet-512) obtained a classification accuracy rate of about 10\% higher than the two-stream fusion model \cite{lgfusion:yu2018two} on the NWPU-RESISC45 dataset \cite{dataset:cheng2017remote}. Furthermore, at the expense of the accuracy of the tolerable range (i.e., about {1\%-2\%}), we can compress the model parameters to 1/64 of the original second-order features. The performance gap between IDCCP models based on full-rank and low-rank is rarely small, and some of them are even only 0.1\%. For example, using ResNet50 to train our model on the NWPU-RESISC45 dataset \cite{dataset:cheng2017remote} (under 20\% training ratio) can achieve 93.76\% and 93.66\% accuracy. Apart from the advanced structure of ResNet50, our IDCCP model also benefits from the orthogonal feature reduction layer and the projection layer (i.e., 1 $\times$ 1 convolution layer), which can effectively remove redundant feature information.
\begin{figure*}
	\centering
	\subfloat[NWPU-RESISC45 \cite{dataset:cheng2017remote} (10\% training ratio)]{\includegraphics[width=3.5in]{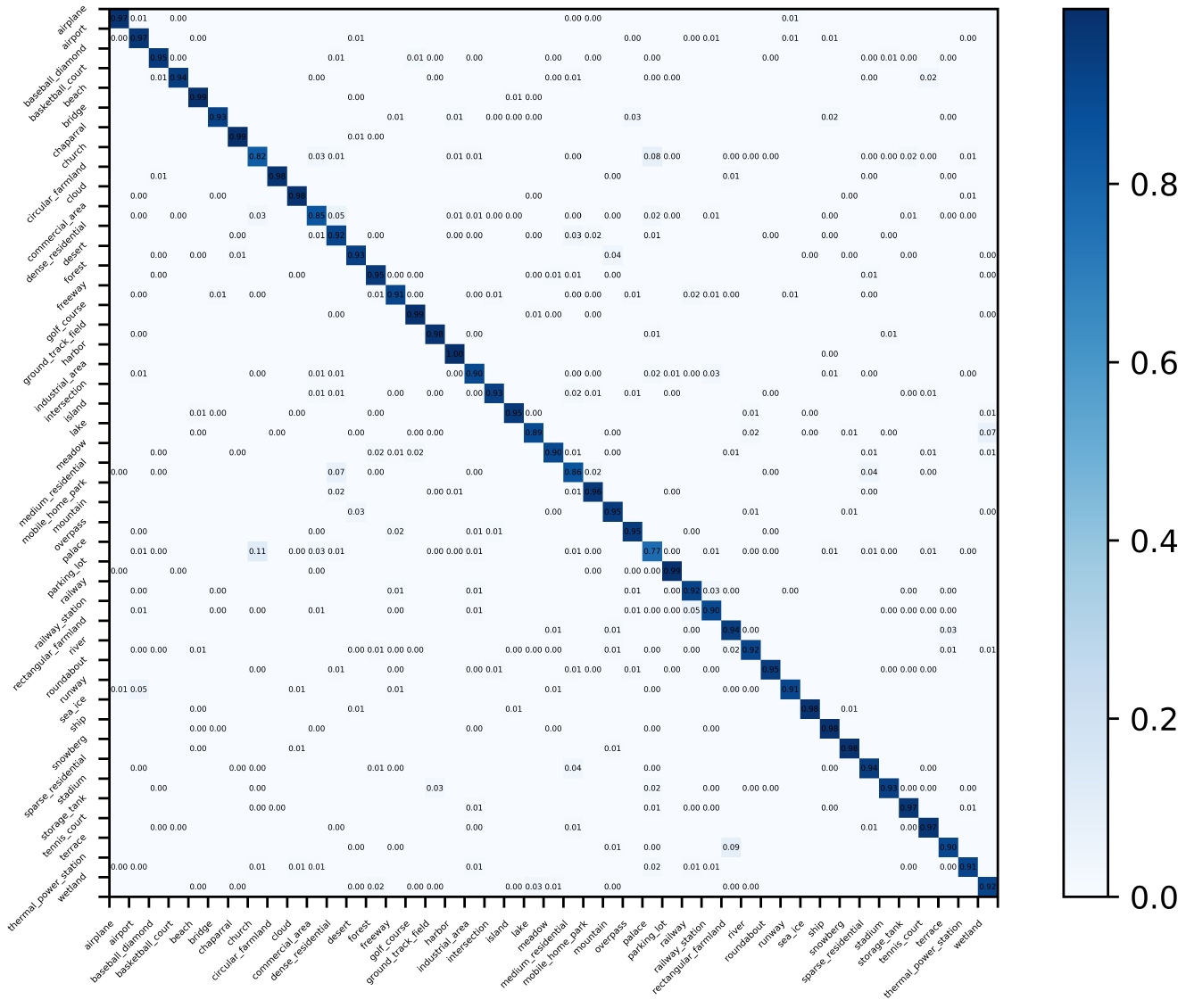}}\label{idccp:nwpu_cm} 
	\subfloat[AID \cite{dataset:xia2017aid} (20\% training ratio)]{\includegraphics[width=3.5in]{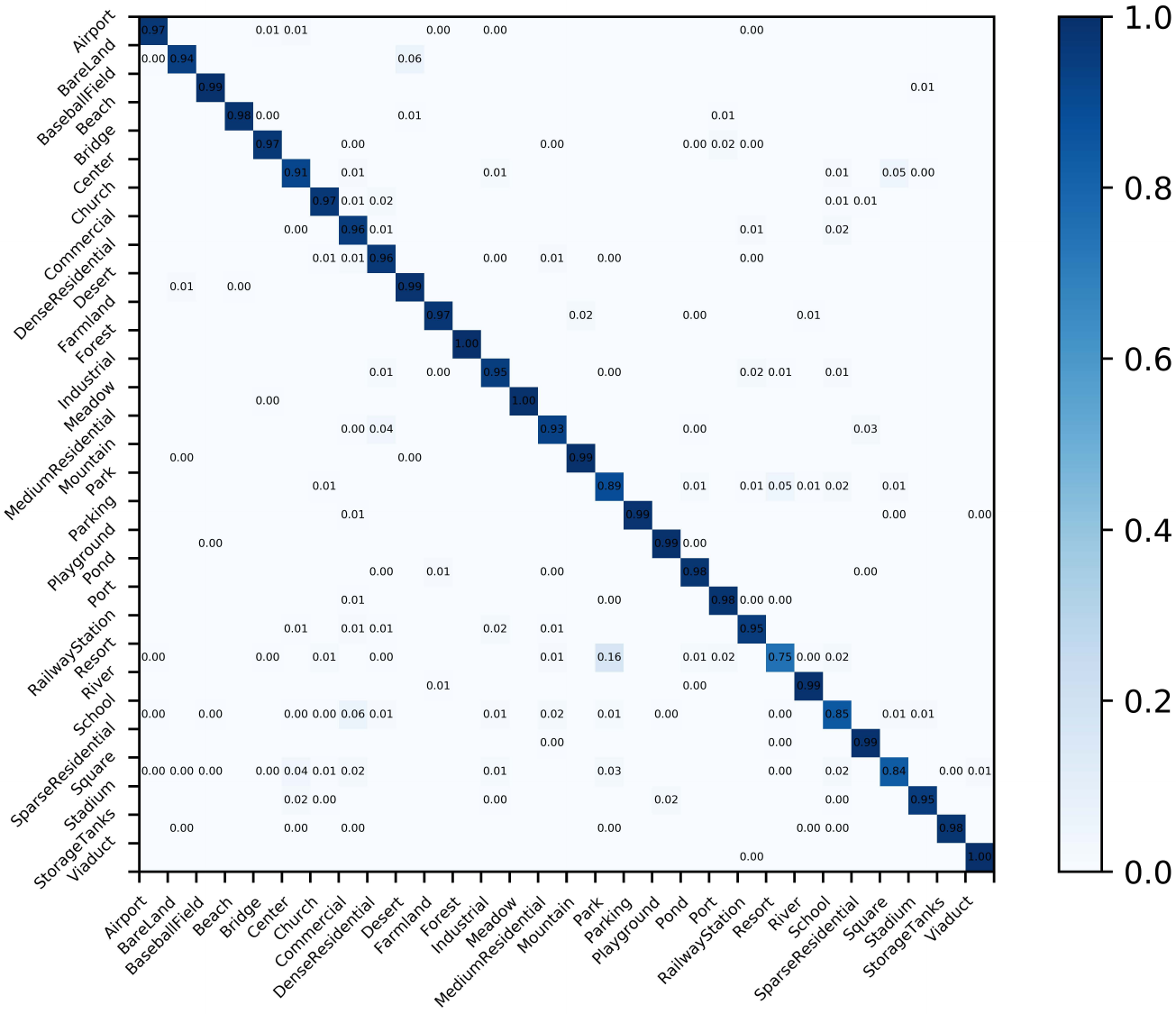}}\label{idccp:aid_cm}\\
	\subfloat[UC Merced \cite{dataset:yang2010bag} (50\% training ratio)]{\includegraphics[width=3.5in]{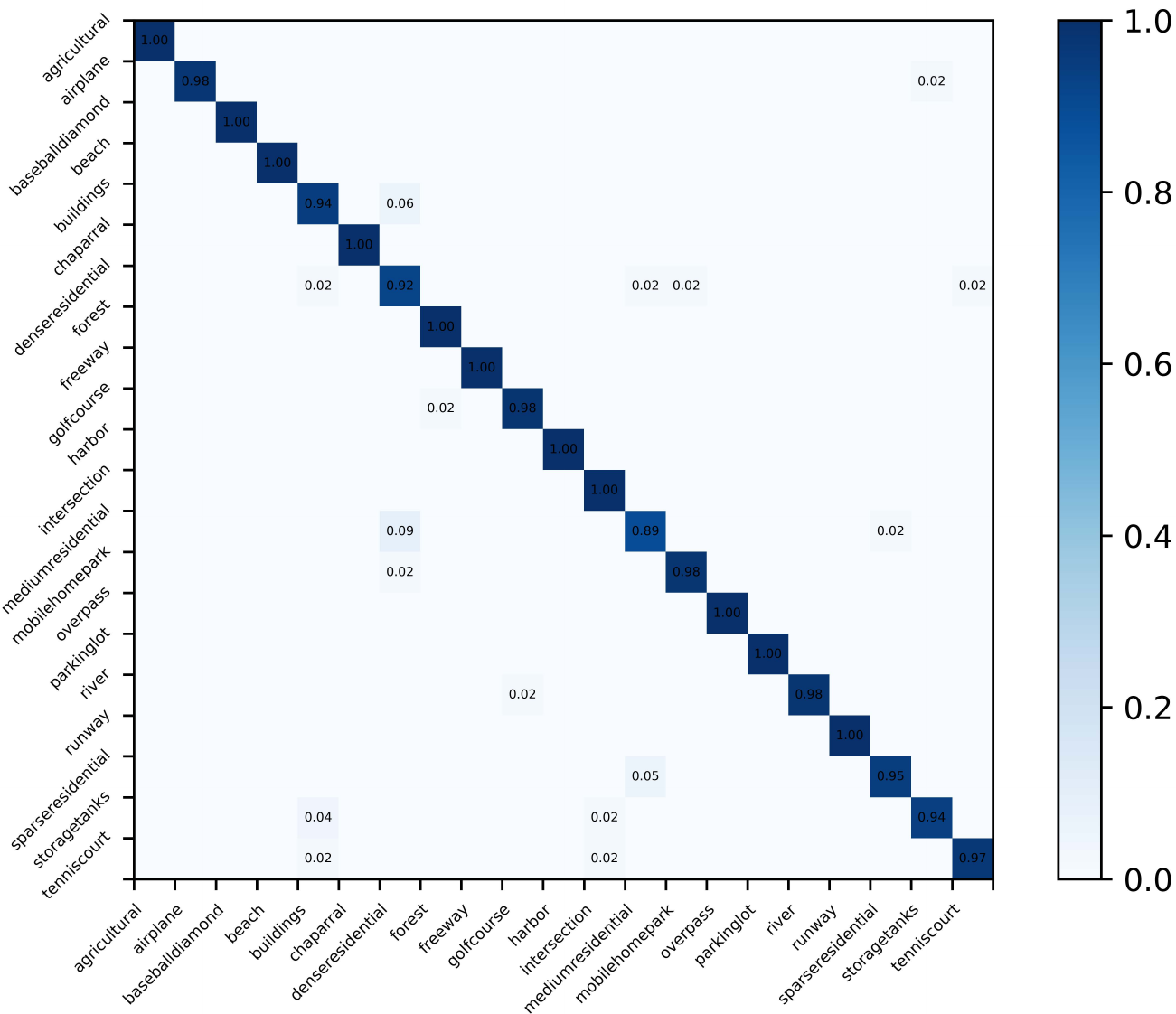}}\label{idccp:uc_merced_cm}
	\subfloat[OPTIMAL-31 \cite{dataset:wang2018scene} (80\% training ratio)]{\includegraphics[width=3.5in]{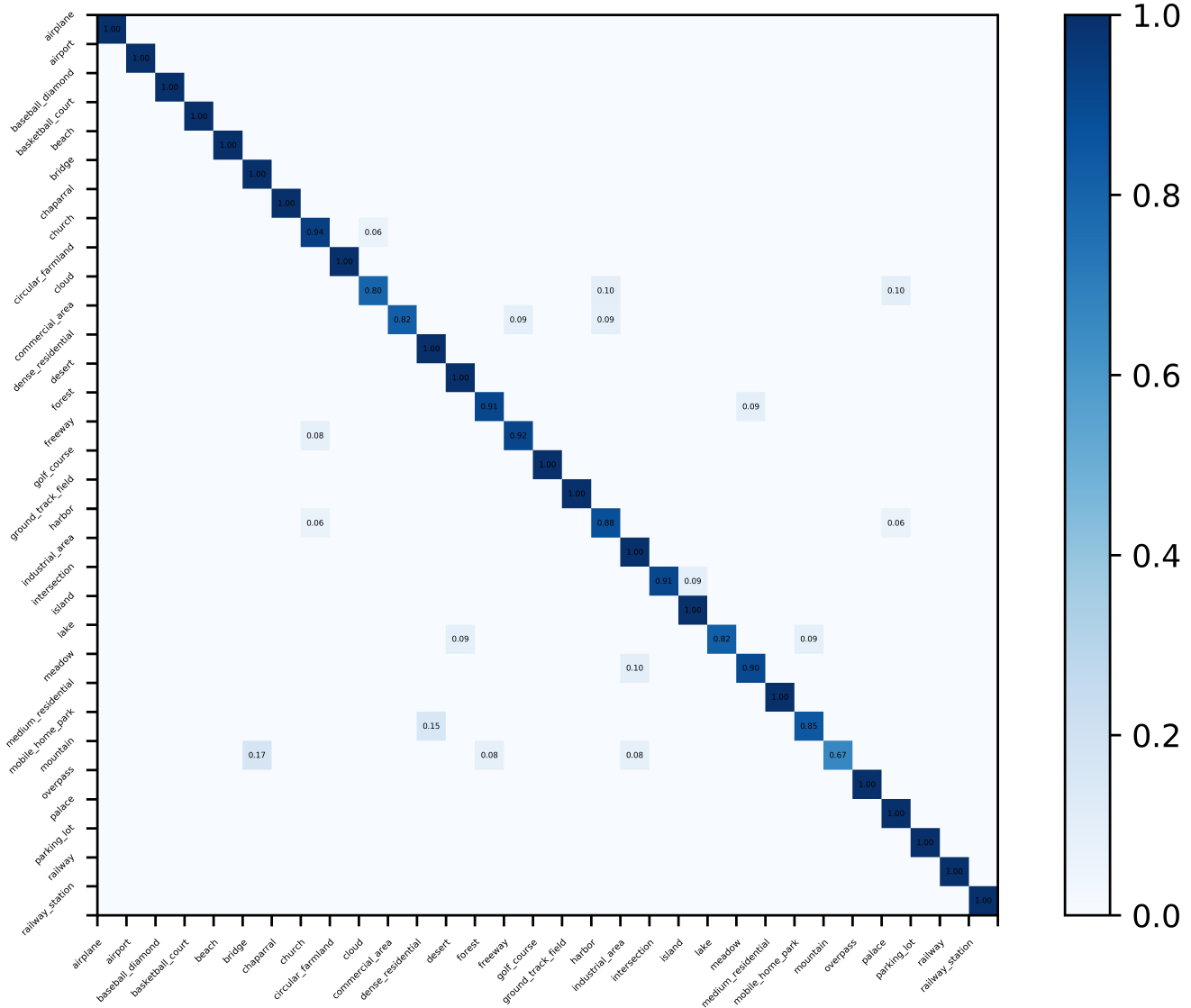}}\label{idccp:optimal31_cm}
	\caption{Results of the confusion matrix on different datasets achieved by our IDCCP model based on ResNet50-512 architecture (not cherry-picking).} 
	\label{idccp:cms} 
\end{figure*}

In addition to comparing overall accuracy, we also provide examples of confusion matrices to show category-level details. For saving space, we randomly selected the results of relatively difficult scenarios on each dataset and displayed in Fig. \ref{idccp:cms}. It can be clearly seen that the darkest color blocks appear on the diagonals of all confusion matrices. On NWPU-RESISC45 \ref{idccp:nwpu_cm}, there exists 40 categories among all 45 categories obtain a classification accuracy rate higher than 90\%. Especially, it reports that the classification accuracy of the \textbf{Palace} and \textbf{Church} categories is 82\% and 77\%, respectively. These two categories are acknowledged as the most visually similar categories in the NWPU-RESISC45 dataset \cite{dataset:cheng2017remote}. Our model can significantly improve the classification accuracy when comparing with the benchmark methods described in {\cite{dataset:cheng2017remote}} (i.e., with improvements of 24\% and 25\%, 7\% and 11\% compared with VGGNet \cite{vgg} and Fine-tuned VGGNet \cite{vgg}, respectively.). Besides, the accuracy that we obtained in these two categories is higher than the two latest technologies by 3\% and 4\% \cite{dcnn}, 2\% and 1\% \cite{rtn}, respectively. On AID \ref{idccp:aid_cm}, the reported accuracy of \textbf{School} is 85\%, which exceeds the algorithm introduced in \cite{dataset:xia2017aid} by 18\%. Due to the high similarity, 16\% of \textbf{Resort} images are misclassified as \textbf{Park}. The classification results of the easily confused \textbf{dense residences}, \textbf{medium-sized residences} and \textbf{sparse residences} are 96\%, 93\% and 99\%, respectively. The sparsest confusion matrix \ref{idccp:uc_merced_cm} is obtained by evaluating the proposed framework on the UC-Merced Land-Use dataset \cite{dataset:yang2010bag} with the training ratio of 50\%. Especially, the accuracy of the \textbf{medium residence} category is 89\% (0.09\% of images are misclassified as \textbf{intensive residence}, which is the only category with an accuracy of less than 90\%. On the OPTIMAL-31 dataset \cite{dataset:wang2018scene}, there are 20 categories of test data that can be classified 100\% correctly \ref{idccp:optimal31_cm}, including those challenging categories, such as \textbf{Church} and \textbf{Island}. 

\subsection{Ablation Study and Analysis}  
\begin{table}[]
	\centering
	\caption{Comparison of classification accuracy and single image inference time. Experiments were conducted on NWPU-RESISC45 dataset \cite{dataset:cheng2017remote} with using 10\% training samples.}
	\scalebox{0.94}{
		\begin{tabular}{c|c|cc|cc}
			\toprule
			\multirow{2}{*}{Networks} & \multirow{2}{*}{Feature Dim.} & \multicolumn{2}{c}{Accuracy (\%)} & \multicolumn{2}{c}{Time (sec/per image)} \\ \cline{3-6}
			&                               & w/ D4           & w/o D4          & w/ D4               & w/o D4             \\\midrule
			\multirow{4}{*}{ResNet50 \cite{resnet}} & 2048                          & -               & 90.02           & -                   & 0.0219             \\
			& 512                           & 91.64           & 90.05           & 0.0768              & 0.0105             \\
			& 64                            & 91.26           & 89.94           & 0.0744              & 0.0093             \\
			& 16                            & 90.78           & 89.83           & 0.0721              & 0.0087             \\\midrule
			\multirow{3}{*}{VGGNet \cite{vgg}}   & 512                           & 91.11           & 89.44           & 0.0324              & 0.0063             \\
			& 64                            & 89.78           & 88.34           & 0.0322              & 0.0059             \\
			& 16                            & 88.62           & 87.21           & 0.0317              & 0.0052  \\\bottomrule          
	\end{tabular}}\label{idccp:dims_exp}
\end{table}
\subsubsection{Compactness and Effectiveness}
In Table \ref{idccp:dims_exp},  extensive
results are listed to show the effect of feature dimensionality and D4 transformation group on classification accuracy and a single image inference time. For a fair comparison, we ensure that all hyperparameters are consistent and then obtain the interaction time of a single image by calculating the ratio of the total test duration to the number of test samples. When we reduce the feature size, the gap in classification accuracy is not significantly enlarged. For example, with ResNet50 architecture \cite{resnet}, the accuracy only decreases by 0.86\% even if we compress the feature space to 1/64 of the original feature space. It is worth noting that our IDCCP model allows features to be compressed into a very compact space (i.e., 16$\times$16) without sacrificing too much accuracy. Interestingly, the classification accuracy is slightly improved when we adopt 1 $\times$ 1 convolution layer to map the CNN feature to a lower feature space. The reason for this phenomenon is that 1 $\times$ 1  convolution can reduce the diversity and redundancy of feature maps, thereby improving the discriminative power of learned feature \cite{sos:wei2018grassmann}. Due to the limited capability of our PC, the accuracy of equipping the D4 transformation group has been omitted. However, this hardly affects us to investigate the effectiveness of the D4 transformation group. At the feature size of 16$\times$16, the IDCCP model based on ResNet50 \cite{resnet} achieved 89.9\% accuracy, which can exceed the full-rank constrained VGGNet model. It not only influenced by the superior structure of ResNet50 \cite{resnet} but also reflects the effectiveness of the projection layer. In addition, ResNet50  \cite{resnet}-based IDCCP model, a single image inference time, only needs about 0.07 and 0.01 seconds for equipping or not equipping the D4 group, respectively. Due to the relatively shallow CNN structure, the inference time reduce by half when using VGGNet \cite{vgg} architecture.
\begin{figure}[]
	\centering
	\includegraphics[width=0.48\textwidth]{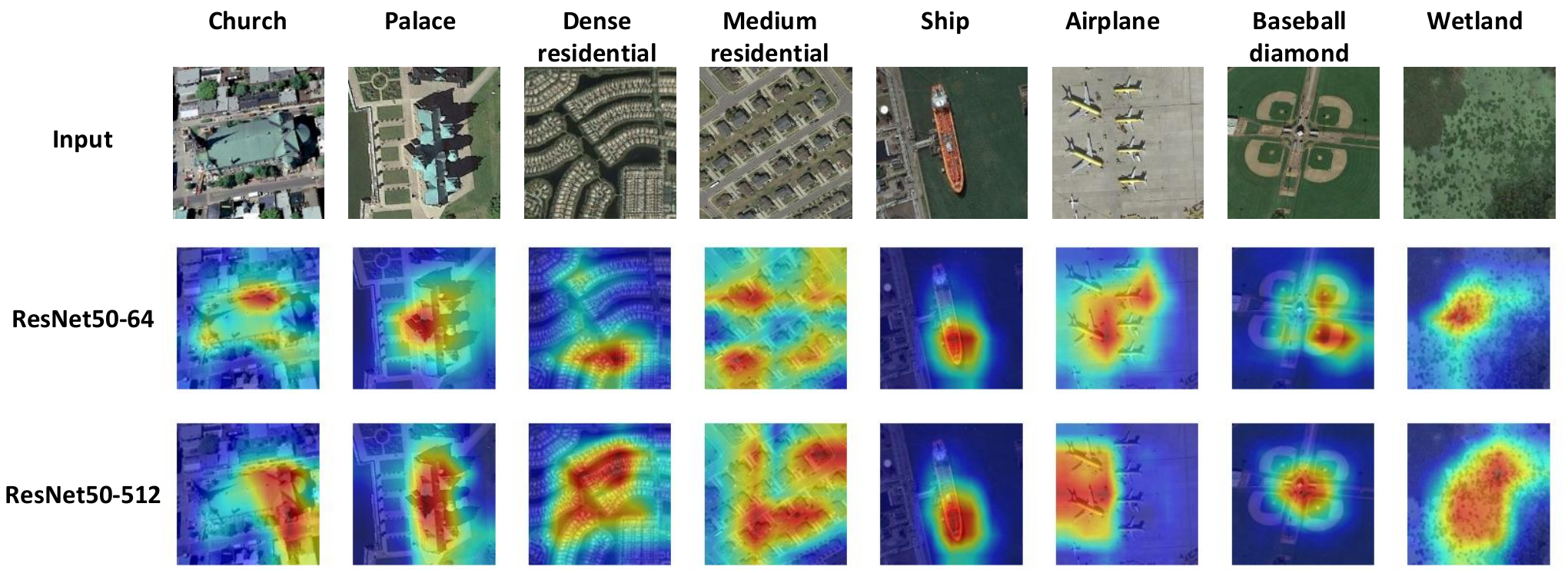}
	\caption{Selected images for qualitative visualisation.}
	\label{idccp:cam}
\end{figure}
\subsubsection{Qualitative Visualization}
Through the comparison of the above experiments, we find that the overall accuracy of the compressed model can be kept at the same par with the uncompressed model. This prompted us to find the evidence from the interpretability of the model. As shown in. \ref{idccp:cam}, we select example images from the NWPU-RESISC45 datasets \cite{dataset:cheng2017remote} and show the corresponding heatmaps by utilizing the Grad-Cam algorithm \cite{cam}. When using ResNet50-64 architecture, we found that our model can focus on small patches that benefit to distinguish subtle differences between visually similar images, such as \textbf{Church} and \textbf{Palace}, \textbf{Dense residential} and \textbf{Medium residential}. Compared with ResNet50-64, the area of attention map is significantly expanded when using ResNet50-512 model. Namely, it allows the model to capture more texture information and could be the reason why ResNet50-512 performs slightly better than ResNet50-64 model.
\section{Conclusion}
In this article, we proposed a novel IDCCP model to handle
the variations in the classification of aerial images. The model benefits from the use of Siamese CNNs to learn the trivial representation of the predefined transformation group. The obtained representation can be deployed to the scenarios of the second-order representation. Meanwhile, we endowed the weight matrix with the form of Stiefel manifold and employed it to reduce the dimensions of the SPD manifold. Finally, the resulting features are flattened to train invariant classifiers. In the future, we will focus on exploring the impact of more complex group structures.
\section{Acknowledgment}
The authors are very grateful to the editor and all reviewers for their valuable comments to improve this article. This work was funded by EPSRC DERC: Digital Economy Research Centre (EP/M023001/1).

\bibliographystyle{IEEEtran}
\bibliography{references}
\end{document}